%% file: main.tex
\documentclass[11pt]{article}

\usepackage[margin=1in]{geometry}
\usepackage[T1]{fontenc}
\usepackage[utf8]{inputenc}
\usepackage{lmodern}
\usepackage[hyphens]{url}
\usepackage{graphicx}
\usepackage{natbib}
\usepackage{caption}
\usepackage{booktabs}
\usepackage{amsfonts}
\usepackage{nicefrac}
\usepackage{microtype}
\usepackage{xcolor}
\usepackage{hyperref}
\hypersetup{
    colorlinks=true,
    citecolor=blue,
    linkcolor=red,
    urlcolor=blue
}
\usepackage{titling}
\thanksmarkseries{arabic}

\usepackage{amsmath}
\usepackage{amssymb,mathtools,amsthm}
\usepackage{bm}
\usepackage{algorithm}
\usepackage{algpseudocode}
\input{math_commands.tex}

\newtheorem{theorem}{Theorem}
\newtheorem{proposition}{Proposition}
\newtheorem{lemma}{Lemma}
\newtheorem{corollary}{Corollary}

\def \TS {T_{\text{stop}}}

\title{Test-Time Scaling via Budgeted Multi-Attribute Verification}

\author{
  Bo Xue\thanks{City University of Hong Kong.
  Email: \texttt{boxue4-c@my.cityu.edu.hk}}
  \quad
  Ji Cheng\thanks{City University of Hong Kong.
  Email: \texttt{J.Cheng@my.cityu.edu.hk}}
  \quad
  Shen-Huan Lyu\thanks{Hohai University.
  Email: \texttt{lvsh@hhu.edu.cn}}
  \quad
  Yuanyu Wan\thanks{Zhejiang University.
  Email: \texttt{wanyy@zju.edu.hk}}
  \quad
  Shuang Qiu{\color{blue}${}^\dag$}\thanks{City University of Hong Kong.
  Email: \texttt{shuanqiu@cityu.edu.hk}}
}

\usepackage[symbol]{footmisc}

\begin{document}

\date{\today}

\footnotetext[2]{Corresponding Author}

\maketitle

\begin{abstract}
Verifying LLM-generated answers under a shared computational budget requires jointly deciding which candidates to inspect and which verification attributes to evaluate. We formulate this problem as multi-attribute good-arm identification under a global budget: each candidate is an arm evaluated along several costly attributes, and the goal is to certify as many candidates as possible whose mean scores exceed the prescribed thresholds on all attributes. We propose \textsc{BMA-GAI}, an algorithm that combines cost-aware arm selection with adaptive sampling of attributes. Every observation serves both to guide adaptive allocation and to support anytime-valid certification, which removes the need for a separate confirmation stage. We establish an asymptotic coverage guarantee for \textsc{BMA-GAI} and derive a matching information-theoretic converse that characterizes the intrinsic complexity of the problem, thereby proving that \textsc{BMA-GAI} is first-order optimal away from critical budget levels. Experiments on synthetic benchmarks and an LLM answer-verification task show that \textsc{BMA-GAI} allocates the verification budget more efficiently and certifies more high-quality candidates than competing methods.

\end{abstract}

\section{Introduction}
\label{sec:intro}

Test-time scaling has emerged as an effective approach to improving the performance of large language models (LLMs) by allocating additional computation at inference time. Existing methods scale test-time computation through repeated sampling and aggregation \citep{wang2023selfconsistency,brown2024largelanguagemonkeys}, search over alternative reasoning trajectories \citep{yao2023tree}, iterative refinement with feedback \citep{madaan2023selfrefine,shinn2023reflexion}, and verifier-guided selection among candidate solutions \citep{cobbe2021training,uesato2022solving,lightman2024verify}. Recent studies further show that reasoning performance can be improved by explicitly controlling test-time computation \citep{snell2025scaling,muennighoff2025s1}, and that the marginal benefit of additional compute depends strongly on both input difficulty and inference strategy \citep{snell2025scaling}. These observations have motivated adaptive allocation across queries, either by predicting the marginal value of additional computation \citep{damani2025learning} or by casting cross-query allocation as a bandit problem \citep{zuo2026strategic,wang2025dynscaling}.

Verification plays a central role in this allocation problem. In candidate-based reasoning, additional generations are useful only if promising responses can be reliably distinguished from incorrect ones \citep{cobbe2021training,lightman2024verify}. Recent work has therefore explored multiple specialized verifiers \citep{lifshitz2025multiagent} and adaptive allocation of verification computation \citep{zhong2026solvedetectverify,singhi2025whentosolve}. These developments suggest viewing verification as a test-time resource whose allocation directly determines how many candidate answers can be reliably certified.

In many applications, however, the validity of an answer cannot be adequately summarized by a single verification score. Different verifiers may assess complementary requirements, such as numerical consistency, logical validity, and other task-specific criteria \citep{lifshitz2025multiagent,kwok2026llmasverifier}. Moreover, verifier calls can differ substantially in both cost and statistical difficulty, which makes the allocation of verification effort nontrivial in its own right \citep{kanarios2024cost}. This heterogeneity enables cost-aware screening: inexpensive verifiers can provide early evidence that a candidate is unlikely to satisfy all criteria, so that expensive verification can be reserved for candidates that remain plausible. At the same time, certification is conjunctive: a candidate can be accepted only after sufficient evidence has been accumulated for \emph{every} required attribute, as in multi-attribute constrained identification \citep{dharod2024grouped}. Under a shared verification budget, the learner must therefore decide both \emph{which candidate} to verify and \emph{which attribute} to inspect next. This leads to our central question: \emph{how should a limited verification budget be allocated across candidates and verification attributes to reliably certify as many candidates as possible?}

We formalize this setting as \emph{multi-attribute good-arm identification under a global verification budget}. Each fixed query--answer pair is modeled as an arm with multiple separately sampleable verification attributes, where evaluating an attribute incurs a known cost and yields a noisy score. An arm is certified only if all of its attributes meet their prescribed thresholds, and the learner aims to certify as many arms as possible within a shared hard budget while controlling the probability of any erroneous certification. This formulation builds on three lines of work: good-arm identification, which seeks arms whose means exceed a prescribed threshold \citep{kano2019goodarm}; multi-attribute constrained bandits with separately observable feasibility criteria \citep{dharod2024grouped}; and cost-aware pure exploration with heterogeneous sampling costs \citep{kanarios2024cost}. In contrast to these settings, we study a budgeted \emph{coverage} objective: the learner allocates verification effort under a shared budget so as to maximize the number of reliably certified arms. The achievable coverage is therefore not fixed in advance, but is jointly determined by the budget and the arms' certification costs.

% In contrast to these settings, we study a budgeted coverage objective: under a shared hard verification budget, the learner must determine which good arms are worth certifying and how to allocate verification effort across their attributes so as to maximize the number of reliably certified arms. Consequently, the achievable coverage is not fixed in advance, but is jointly determined by the available budget and the instance-dependent certification costs of the candidate arms.

% In contrast to these settings, our goal is neither to identify a single preferred arm nor to attain a prespecified number of discoveries. Instead, the number of arms that can be certified is determined adaptively by the available verification budget and the unknown certification difficulty of each arm.

To answer this question, we propose \textsc{BMA-GAI}, an adaptive allocation algorithm that simultaneously learns which candidates are inexpensive to certify and how to distribute verification effort across their unresolved attributes. Our main contributions are summarized as follows:
\begin{itemize}
\item We formulate \emph{multi-attribute good-arm identification under a global verification budget}, in which each candidate must satisfy multiple separately observable criteria with heterogeneous evaluation costs, and the goal is to maximize reliable certification coverage.

\item We develop \textsc{BMA-GAI}, which adaptively allocates verification effort across candidates and their unresolved attributes via cost-aware selection, and we establish an achievable coverage guarantee showing that it reliably certifies as many good candidates as the available budget permits.

\item We characterize the fundamental complexity of reliable coverage: certifying $\ell$ good arms requires an aggregate complexity of $T_\ell^*(\mu)$ (up to a $\log(1/\delta)$ factor), which is determined by the $\ell$ least costly good arms. A matching information-theoretic converse shows that \textsc{BMA-GAI} attains this frontier to first order away from critical budget thresholds.

% \item We characterize the fundamental complexity of reliable coverage and derive a matching information-theoretic converse. The resulting characterization shows that optimal coverage is determined by the aggregate certification complexities of the least costly good arms and establishes the first-order optimality of \textsc{BMA-GAI} away from critical budget thresholds.

\item We evaluate \textsc{BMA-GAI} on synthetic benchmarks and an LLM answer-verification task, demonstrating the benefits of adaptively allocating verification effort across candidates and heterogeneous verification attributes under limited budgets.

\end{itemize}

\section{Related Work}
\label{sec:related-work}

\paragraph{Test-Time Scaling for LLMs.}
Test-time scaling has become a prominent paradigm for improving LLM reasoning without modifying model parameters. One class of methods increases inference-time computation by generating multiple reasoning trajectories and then aggregating or selecting among them. Representative approaches include self-consistency and its adaptive or voting-based variants \citep{wang2023selfconsistency,chen2023universal, aggarwal2023adaptiveconsistency,xue2023dynamicvoting}, as well as repeated-sampling and Best-of-$N$ strategies that select promising responses from larger candidate pools \citep{brown2024largelanguagemonkeys,snell2025scaling,wang2025stbon}.
A second line of work spends additional computation on searching over or iteratively refining reasoning trajectories, including Tree-of-Thoughts-style search \citep{yao2023tree}, self-refinement \citep{madaan2023selfrefine}, and feedback-driven iterative reasoning \citep{shinn2023reflexion}.
A third class devotes inference-time computation to evaluating generated candidates, using verification signals to assess complete solutions or intermediate reasoning steps \citep{cobbe2021training,uesato2022solving,lightman2024verify}.
More broadly, recent studies show that the benefit of additional test-time computation depends strongly on the task, input difficulty, and inference strategy \citep{snell2025scaling,muennighoff2025s1}, motivating a shift from uniformly increasing inference compute toward allocating it selectively.

Recent work has increasingly explored \emph{adaptive} test-time scaling, including input-adaptive allocation of decoding computation \citep{damani2025learning}, bandit-based budget allocation across queries \citep{zuo2026strategic,wang2025dynscaling}, and adaptive verification through multiple verifiers or dynamic allocation of verification effort \citep{lifshitz2025multiagent,zhong2026solvedetectverify,singhi2025whentosolve}. In contrast, we study how to allocate a verification budget across a fixed set of candidate answers and multiple verification attributes, with the objective of maximizing reliable certification coverage.

\paragraph{Good-Arm Identification and Cost-Aware Pure Exploration.}
A broad literature on bandit pure exploration studies threshold-based identification \citep{locatelli2016thresholding,kano2019goodarm,cho2025reward,jiang2025multithreshold},
problems with multiple correct answers \citep{Degenne:2019},
structured feasibility constraints \citep{dharod2024grouped,lardy2025constrained,cai2026feasibility,mukherjee2026fixedbudget},
and heterogeneous sampling costs or resource constraints \citep{kanarios2024cost,li2024resource}.
Thresholding bandits classify arms according to whether their mean rewards exceed a prescribed threshold \citep{locatelli2016thresholding}, whereas good-arm identification (GAI) aims to sequentially output satisfactory arms as soon as they can be certified \citep{kano2019goodarm}. Multiple-answer pure exploration allows several outputs to be simultaneously correct and studies how sampling should be allocated among competing valid answers \citep{Degenne:2019}. More recent work extends these ideas to anytime-valid and nonparametric GAI \citep{cho2025reward} and to settings with multiple threshold constraints \citep{jiang2025multithreshold}.

A complementary line of work incorporates explicit resource constraints and heterogeneous sampling costs into pure exploration \citep{kanarios2024cost,li2024resource,dharod2024grouped,lardy2025constrained,cai2026feasibility,mukherjee2026fixedbudget}.
Cost-aware best-arm identification accounts for arm-dependent sampling costs \citep{kanarios2024cost}, while resource-constrained formulations impose global limits on the resources consumed during identification \citep{li2024resource}. Closest to our feedback model, constrained grouped bandits consider arms composed of separately observable attributes subject to conjunctive feasibility requirements \citep{dharod2024grouped}, and subsequent work studies related constrained and fixed-budget variants \citep{lardy2025constrained,cai2026feasibility,mukherjee2026fixedbudget}. Our setting combines these two lines of research but pursues a different goal: certifying as many arms as possible under a shared hard budget, where each certified arm must satisfy all prescribed attribute-wise thresholds.

\section{Preliminaries}
\label{sec:prelim}
We formulate test-time answer verification as a \emph{multi-attribute bandit problem} under a global verification budget $B>0$. There are $Q$ candidate arms indexed by $[Q]\coloneqq\{1,\ldots,Q\}$. In the LLM setting, each arm $q$ corresponds to a fixed query--answer pair, and $a_q$ denotes the associated candidate answer. Each arm has $M$ separately sampleable attributes indexed by $[M]\coloneqq\{1,\ldots,M\}$, and sampling attribute $m$ of arm $q$ amounts to pulling the arm--attribute pair $(q,m)$. Each attribute corresponds to a fixed verification procedure (a verifier) that assesses a particular aspect of answer validity. A fresh invocation of verifier $m$ on candidate answer $a_q$ incurs a known cost $c_m>0$ and returns a stochastic verification score with unknown mean $\mu_{q,m}\in[0,1]$, normalized so that larger values indicate stronger evidence of validity. Each attribute $m$ is associated with a prescribed acceptance threshold $\xi_m\in(0,1)$.

\paragraph{Budgeted bandit interaction.}
At each round $t$, the learner adaptively selects an arm--attribute pair $(q_t,m_t)\in[Q]\times[M]$ and observes a stochastic verification score $Y_t$. Let $\mathcal F_{t-1}$ denote the history available before round $t$, including the learner's internal randomization, so that $(q_t,m_t)$ is $\mathcal F_{t-1}$-measurable. We assume that each new observation satisfies
\begin{equation}
\begin{aligned}
\mathbb E[Y_t\mid\mathcal F_{t-1}]
=\mu_{q_t,m_t},
\quad
\mathbb E\!\left[
\exp\!\left(\lambda(Y_t-\mu_{q_t,m_t})\right)
\,\middle|\,\mathcal F_{t-1}
\right]
\le \exp\!\left(\frac{\lambda^2}{2}\right),
\quad \forall \lambda\in\mathbb R.
\end{aligned}
\label{eq:subgaussian-verifier}
\end{equation}
Thus, repeated pulls of the same arm--attribute pair yield fresh verifier observations, and the learner may adapt its sampling decisions to all previously observed evidence. Since each pull of attribute $m$ incurs cost $c_m$, the cumulative cost after $t$ observations is
\begin{equation*}
C_t
=
\sum_{s=1}^{t} c_{m_s}.
\end{equation*}
The learner is given a global verification budget $B>0$. Because verification costs may vary across attributes, the budget constrains the total incurred cost rather than merely the number of verifier calls. The learner must therefore decide adaptively both which arm--attribute pairs to sample and how to distribute the available budget among them.

\paragraph{Good arms.}
An arm is deemed good if all of its attribute means meet their respective acceptance thresholds. Accordingly, the set of good arms is
\begin{equation*}
\mathcal G
=
\left\{
q\in[Q]:
\mu_{q,m}\ge\xi_m,\ \forall m\in[M]
\right\}.
\end{equation*}
In LLM verification, $q\in\mathcal G$ means that the candidate answer $a_q$ satisfies all verification criteria. This definition induces a fundamental asymmetry: violating a single threshold suffices to disqualify an arm, whereas certification requires evidence that \emph{all} thresholds are met. Consequently, the verification budget should be allocated adaptively, rather than uniformly, across candidate answers and attributes. To characterize both the algorithmic guarantee and the intrinsic problem complexity, we assume $\mu_{q,m}\neq\xi_m$ for all $(q,m)\in[Q]\times[M]$.

\paragraph{Budgeted reliable coverage.}
For a prescribed confidence level $\delta\in(0,1)$, a policy $\pi$ consists of an adaptive sampling rule, a stopping time $\TS$ taking nonnegative integer values, and an $\mathcal F_{\TS}$-measurable output set $\widehat{\mathcal S}_{\pi}(B)\subseteq[Q]$. The policy may depend on the budget $B$, the confidence level $\delta$, and the known costs and thresholds. Our objective is to maximize the expected number of returned arms subject to family-wise reliability and a hard verification budget:
\begin{equation*}
\begin{aligned}
\sup_{\pi}\,
\mathbb E_{\boldsymbol\mu,\pi}
\left[
\left|\widehat{\mathcal S}_{\pi}(B)\right|
\right]
\qquad
\text{s.t.}\,
\mathbb P_{\boldsymbol\mu,\pi}
\left(\widehat{\mathcal S}_{\pi}(B)\subseteq\mathcal G\right)\ge 1-\delta,
\quad
C_{\TS}\le B.
\end{aligned}
\end{equation*}
Here, $\boldsymbol\mu=(\mu_{q,m}:(q,m)\in[Q]\times[M])$ denotes the matrix of unknown means. The objective is thus to maximize reliable coverage under a fixed cumulative-cost budget. Since the achievable coverage depends on both the budget and the difficulty of the underlying instance, effective verification calls for adaptive allocation across arms and attributes.

\section{Algorithm}
\label{sec:track-algorithm}

We now present \textsc{BMA-GAI}, an adaptive allocation algorithm for multi-attribute good-arm identification under a global verification budget. The algorithm maintains a set of certified arms and adaptively allocates verifier calls among the remaining arms. Each observation serves both to guide subsequent allocation and to accumulate evidence for certification, and the algorithm terminates once all arms are certified or no further verifier call is affordable. The full procedure is given in Algorithm~\ref{alg:bgai-track}.

Let $D=QM$. For each $t=0,1,2,\ldots$, let $\mathcal S_t$, $\mathcal Q_t$, $\mathcal M_t$, and $C_t$ denote, respectively, the set of certified arms, the set of active arms, the set of currently affordable attributes, and the cumulative verification cost after $t$ observations. For each arm--attribute pair $(q,m)$, let $N_{q,m}(t)$, $\widehat\mu_{q,m}(t)$, and $L_{q,m}(t)$ denote its sample count, empirical mean, and lower certificate. Time indices are suppressed in Algorithm~\ref{alg:bgai-track} for notational simplicity. \textsc{BMA-GAI} initializes $\mathcal S_0=\varnothing$, $\mathcal Q_0=[Q]$, $C_0=0$, and $\mathcal M_0=\{m\in[M]:C_0+c_m\le B\}$, and sets $N_{q,m}(0)=\widehat\mu_{q,m}(0)=L_{q,m}(0)=0$ for each $(q,m)$.

To describe the sampling rule, we first specify the certification state maintained by the algorithm. For each pair $(q,m)$, the lower certificate after $t$ observations is
\begin{equation}
L_{q,m}(t)
=\max\left\{
L_{q,m}(t-1),
\widehat\mu_{q,m}(t)
-
\sqrt{
\frac{2}{N_{q,m}(t)}
\log\frac{4D N_{q,m}^2(t)}{\delta}
}
\right\},
\quad
L_{q,m}(0)=0,
\label{eq:track-certificate}
\end{equation}
whenever $N_{q,m}(t)>0$; otherwise, it remains unchanged. Based on these certificates, the set of unresolved attributes of arm $q$ is
\begin{equation*}
\mathcal U_q(t)
=
\{m\in[M]:L_{q,m}(t)<\xi_m\}.
\end{equation*}
Arm $q$ is certified once $\mathcal U_q(t)=\varnothing$, i.e., once the lower certificate of every attribute reaches its threshold. Certified arms are then removed from the active set, while all remaining arms continue to participate in sampling.

\begin{algorithm}[tb]
\caption{\textsc{BMA-GAI}: Budgeted Multi-Attribute Good-Arm Identification}
\label{alg:bgai-track}
\begin{algorithmic}[1]
\Require Arms $\mathcal A=\{a_q\}_{q=1}^Q$; thresholds $\{\xi_m\}_{m=1}^M$; costs $\{c_m\}_{m=1}^M$; budget $B>0$; $\delta\in(0,1/2)$
\State Initialize $\mathcal S\gets\varnothing$, $\mathcal Q\gets[Q]$, $t\gets0$, $C\gets0$, $\mathcal M\gets\{m\in[M]:C+c_m\le B\}$
\State Initialize $(N_{q,m},\widehat\mu_{q,m},L_{q,m})\gets(0,0,0)$,
$\forall(q,m)\in[Q]\times[M]$
\While{$\mathcal Q\neq\varnothing$ \textbf{and} $\mathcal M\neq\varnothing$}
    \If{Eq.~(\ref{eq:track-exploration}) holds}
        \State Choose
        $
        (q,m)\in
        \arg\min_{q\in\mathcal Q,\;m\in\mathcal M}
        N_{q,m}
        $
    \Else
        \State Compute $\ell,u,H^-,H^+$ 
        using Eqs.~(\ref{eq:track-allocation-interval})--(\ref{eq:track-cost-interval})
        \If{$\min_{q\in\mathcal Q}H_q^+=\infty$}
            \State Choose
            $
            (q,m)\in
            \arg\min_{q\in\mathcal Q,\;m\in\mathcal M}
            N_{q,m}
            $
        \Else
            \State Choose $q$
            using Eq.~(\ref{eq:track-stable-target}) and $m$ using Eq.~(\ref{eq:track-attribute})
        \EndIf
    \EndIf
    \State Obtain one fresh observation $Y$ from $(q,m)$
    \State $C\gets C+c_m$, $t\gets t+1$
    \State Update $N_{q,m}$, $\widehat\mu_{q,m}$, and $L_{q,m}$
    using Eq.~(\ref{eq:track-certificate})
    \State $\mathcal C
    \gets
    \{q\in\mathcal Q:
    L_{q,m}\ge\xi_m,\ \forall m\in[M]\}$
    \State $\mathcal S\gets\mathcal S\cup\mathcal C$,
    $\mathcal Q\gets\mathcal Q\setminus\mathcal C$
    \State $\mathcal M\gets\{m\in[M]:C+c_m\le B\}$
\EndWhile
\State \Return $\widehat{\mathcal S}=\mathcal S$; abstain on all remaining arms
\end{algorithmic}
\end{algorithm}

\paragraph{Sampling decision.}

Given the current certification state, \textsc{BMA-GAI} selects the next arm--attribute pair $(q_{t+1},m_{t+1})$ from the active arms and currently affordable attributes, which includes three cases:

\emph{(i) Sparse exploration.}
The algorithm first checks whether some currently available arm--attribute pair is under-explored. Specifically, if
\begin{equation}
\min_{q\in\mathcal Q_t,\,m\in\mathcal M_t}
N_{q,m}(t)
<
\frac{\sqrt{t+1}}{D},
\label{eq:track-exploration}
\end{equation}
it selects a least-sampled pair in $\mathcal Q_t\times\mathcal M_t$ as $(q_{t+1},m_{t+1})$. This sparse exploration prevents the allocation rule from over-relying on early estimates and ensures that all currently available arm--attribute pairs continue to be explored.

\emph{(ii) Cost-aware arm screening.}
If the sparse-exploration condition is not triggered, the algorithm uses the available observations to assess which active arms are potentially inexpensive to certify. It first constructs the allocation intervals
\begin{equation}
\begin{aligned}
\ell_{q,m}(t)
=
\widehat\mu_{q,m}(t)-\sqrt{\frac{8\log(t+1)}{N_{q,m}(t)}},\quad
u_{q,m}(t)
=
\widehat\mu_{q,m}(t)+\sqrt{\frac{8\log(t+1)}{N_{q,m}(t)}}.
\end{aligned}
\label{eq:track-allocation-interval}
\end{equation}
When $N_{q,m}(t)=0$, we set $\ell_{q,m}(t)=0$ and $u_{q,m}(t)=1$, based on the known range $\mu_{q,m}\in[0,1]$. These intervals are used only for allocation; they do not enter the certification rule and introduce no additional confidence parameter.

Let $\ell_q(t)=(\ell_{q,m}(t))_{m=1}^M$ and $u_q(t)=(u_{q,m}(t))_{m=1}^M$ denote the vectors of interval endpoints. For a vector $z\in\mathbb R^M$, define the certification-cost surrogate
\begin{equation*}
H(z)
=
\begin{cases}
\displaystyle
\sum_{m=1}^M \frac{2c_m}{(z_m-\xi_m)^2},
&
z_m>\xi_m,\ \forall m\in[M],\\[1ex]
+\infty,
&
\text{otherwise}.
\end{cases}
\end{equation*}
The algorithm then computes the operational cost indices
\begin{equation}
\bigl(H_q^-(t),H_q^+(t)\bigr)
=
\begin{cases}
\bigl(H(u_q(t)),H(\ell_q(t))\bigr),
&\mathcal U_q(t)\cap\mathcal M_t\neq\varnothing,\\
(+\infty,+\infty),
&\mathcal U_q(t)\cap\mathcal M_t=\varnothing.
\end{cases}
\label{eq:track-cost-interval}
\end{equation}
Here, $H_q^-(t)$ and $H_q^+(t)$ are optimistic and conservative estimates, respectively, of the verification cost required to certify arm $q$. The value $+\infty$ in the second case serves only as an allocation-feasibility device that prevents target selection from choosing an arm with no currently affordable unresolved attribute; it neither constitutes a permanent rejection nor implies that the arm has infinite intrinsic statistical complexity. In the budget-free reference process used in our analysis (Appendix~\ref{app:track-proofs}), $\mathcal M_t=[M]$ for all $t$, so the operational and raw indices coincide for every active uncertified arm.

If $\min_{q\in\mathcal Q_t} H_q^+(t)=\infty$, then either the current observations do not yet yield a finite conservative cost estimate for any active arm, or no active arm has an affordable unresolved attribute. In this case, \textsc{BMA-GAI} falls back to sparse exploration and selects a least-sampled pair in $\mathcal Q_t\times\mathcal M_t$.

\emph{(iii) Stable target and attribute allocation.}
Once at least one active arm admits a finite conservative cost estimate, the algorithm constructs the set of potentially least-cost arms and selects the one with the smallest index as the target:
\begin{equation}
q_{t+1}=\min\mathcal P_t,
\quad
\mathcal P_t
=
\left\{
q\in\mathcal Q_t:
H_q^-(t)\le \min_{i\in\mathcal Q_t} H_i^+(t)
\right\}.
\label{eq:track-stable-target}
\end{equation}
Here, $\mathcal P_t$ consists of the active arms that cannot yet be ruled out as the least costly to certify. Selecting the smallest index provides a stable tie-breaking rule and avoids unnecessary target switching caused by small fluctuations in the empirical estimates.

Given the target arm $q_{t+1}$, the algorithm then allocates verification effort across its unresolved attributes. For each $m\in\mathcal U_{q_{t+1}}(t)$, define the optimistic sample requirement $\widetilde v_{q_{t+1},m}(t)=2/(u_{q_{t+1},m}(t)-\xi_m)^2$. By construction of the operational index, $\mathcal U_{q_{t+1}}(t)\cap\mathcal M_t$ is nonempty. Moreover, finiteness of $H_{q_{t+1}}^-(t)$ implies $u_{q_{t+1},m}(t)>\xi_m$ for every unresolved attribute $m$, so the following selection rule is well defined:
\begin{equation}
m_{t+1}
\in
\arg\min_{m\in\mathcal U_{q_{t+1}}(t)\cap\mathcal M_t}
\frac{N_{q_{t+1},m}(t)}
{\widetilde v_{q_{t+1},m}(t)}.
\label{eq:track-attribute}
\end{equation}
Thus, verification is directed to the unresolved attribute whose current sample count is smallest relative to its requirement. Verifier costs enter target-arm selection through $H_q^-(t)$ and $H_q^+(t)$, whereas Eq.~(\ref{eq:track-attribute}) balances statistical evidence across the unresolved attributes of the target arm.

\paragraph{Observation and certification.}
After selecting $(q_{t+1},m_{t+1})$, \textsc{BMA-GAI} obtains a fresh verifier observation $Y_{t+1}$ and updates the cumulative cost to $C_{t+1}=C_t+c_{m_{t+1}}$. Based on the updated certificates, the algorithm identifies the set of
newly certified arms
\begin{equation*}
    \mathcal C_{t+1}
    =
    \left\{
        q\in\mathcal Q_t:
        L_{q,m}(t+1)\ge\xi_m,\ \forall m\in[M]
    \right\}.
\end{equation*}
The newly certified arms are added to the certified set and removed from the active set:
\begin{equation*}
    \mathcal S_{t+1}=\mathcal S_t\cup\mathcal C_{t+1},
    \quad
    \mathcal Q_{t+1}=\mathcal Q_t\setminus\mathcal C_{t+1}.
\end{equation*}
Thus, $\mathcal S_t$ records all arms certified within the first $t$ observations and grows monotonically. In contrast, an arm that does not yet meet the certification condition remains in $\mathcal Q_t$ and may receive further verification in subsequent rounds. Certification is therefore the only irreversible decision made by the algorithm; unreturned arms are never permanently rejected.

\paragraph{Budget update and termination.}
After each observation, \textsc{BMA-GAI} also updates the set of affordable attributes according to the remaining budget:
\begin{equation*}
    \mathcal M_{t+1}
    =
    \left\{
        m\in[M]:
        C_{t+1}+c_m\le B
    \right\}.
\end{equation*}
In other words, an attribute is excluded from further consideration once its cost exceeds the remaining budget. The next round then proceeds with the active set $\mathcal Q_{t+1}$ and the affordable set $\mathcal M_{t+1}$. This sampling--certification cycle continues as long as $\mathcal Q_t\neq\varnothing$ and $\mathcal M_t\neq\varnothing$. The algorithm stops at time $\TS$ when either all arms have been certified ($\mathcal Q_{\TS}=\varnothing$) or no further verifier call is affordable ($\mathcal M_{\TS}=\varnothing$). It then returns the accumulated certified set $\widehat{\mathcal S}(B)=\mathcal S_{\TS}$ and abstains on the remaining active arms. Hence, the number of returned answers is not fixed in advance, but is determined adaptively by the available budget and the statistical difficulty of the instance.

\section{Theoretical Guarantees}
\label{sec:track-theory}

In this section, we characterize both the performance of \textsc{BMA-GAI} and the intrinsic difficulty of multi-attribute certification under a shared verification budget. Together, matching achievability and converse bounds identify the first-order optimal coverage.

% We evaluate algorithmic performance in terms of reliability, which requires all returned arms to be good, and coverage, which measures the number of good arms certified before the budget is exhausted. Problem difficulty is captured by characteristic certification costs that combine verifier costs with the corresponding gaps to their acceptance thresholds. We first establish anytime reliability and compliance with the hard budget constraint. We then derive an achievable coverage guarantee under conditionally sub-Gaussian feedback and an information-theoretic upper bound on coverage under independent Gaussian observations. The matching bounds characterize the first-order optimal coverage away from the characteristic-complexity thresholds.

\begin{proposition}[Anytime reliability and hard-budget feasibility]
\label{thm:track-safety}
For any fixed budget $B>0$, \textsc{BMA-GAI} satisfies $C_{\TS}\le B$ almost surely and
\begin{equation}
    \mathbb P_\mu\!\left(
        \mathcal S_t\subseteq\mathcal G
        \text{ for all } t\le\TS
    \right)\ge 1-\delta.
    \label{eq:track-safety}
\end{equation}
\end{proposition}
Proposition~\ref{thm:track-safety} thus provides a time-uniform certification guarantee: with probability at least $1-\delta$, every arm certified at any time before termination is good, despite the adaptive allocation across arms and attributes. Moreover, the cumulative verification cost never exceeds the prescribed budget, so reliability is attained without violating the hard budget constraint.

To characterize both the algorithmic guarantee and the intrinsic problem complexity, for $q\in\mathcal G$, define
\begin{equation*}
\begin{aligned}
    \Delta_{q,m}=\mu_{q,m}-\xi_m,
    \quad
    h_{q,m}=\frac{2c_m}{\Delta_{q,m}^2},
    \quad
    h_q=\sum_{m=1}^M h_{q,m}.
\end{aligned}
\end{equation*}
Here, $h_{q,m}$ quantifies the complexity of certifying attribute $m$ of arm $q$, accounting for both the evaluation cost $c_m$ and the threshold gap $\Delta_{q,m}$, and $h_q$ aggregates this complexity over all attributes. Building on these arm-wise complexities, define, for $1\le \ell\le |\mathcal G|$,
\begin{equation}
T_\ell^*(\mu)
=
\min_{\substack{\mathcal S\subseteq\mathcal G\\|\mathcal S|=\ell}}
\sum_{q\in\mathcal S} h_q,
\quad
T_0^*(\mu)=0.
\label{eq:track-characteristic-closed}
\end{equation}
Thus, $T_\ell^*(\mu)$ is the minimum aggregate complexity required to certify any $\ell$ good arms, attained by the $\ell$ good arms with the smallest $h_q$. The following theorem shows that this benchmark directly governs the coverage achieved by \textsc{BMA-GAI} under a global budget.

\begin{theorem}[Achievable coverage under a global budget]
\label{thm:track-coverage-achievable}
Fix an analysis slack $\varepsilon\in(0,1)$ and a normalized budget $b>0$, and define $\underline K_\varepsilon(b)=\max\!\left\{\ell\in\{0,\ldots,|\mathcal G|\}:(1+\varepsilon)T_\ell^*(\mu)<b\right\}$. Then, with $B_\delta=b\log(1/\delta)$, \textsc{BMA-GAI} satisfies
\begin{equation*}
    \mathbb P_\mu\!\left(
        \widehat{\mathcal S}(B_\delta)\subseteq\mathcal G,
        \quad
        |\widehat{\mathcal S}(B_\delta)|
        \ge\underline K_\varepsilon(b)
    \right)\longrightarrow1
    \quad\text{as }\delta\downarrow0.
\end{equation*}
\end{theorem}
Theorem~\ref{thm:track-coverage-achievable} provides an asymptotic performance guarantee for \textsc{BMA-GAI} at normalized budget $b$: with probability tending to one, the algorithm returns only good arms and certifies at least $\underline K_\varepsilon(b)$ of them. In other words, \textsc{BMA-GAI} asymptotically attains every coverage level $\ell$ whose aggregate certification complexity satisfies $(1+\varepsilon)T_\ell^*(\mu)<b$. The proof first bounds the certification costs in a budget-free reference process and then couples the finite-budget algorithm with this process for as long as every attribute remains affordable.

We next show that, under Gaussian observations, the same quantity $T_\ell^*(\mu)$ also characterizes the intrinsic complexity of the problem. Specifically, consider independent unit-variance observations $\mathcal N(\mu_{q,m},1)$, and let $\Pi_\delta(B)$ denote the class of policies that are uniformly $\delta$-reliable over this Gaussian model class and satisfy the hard budget $B$ on every admissible instance.

\begin{theorem}[Information-theoretic coverage limit]
\label{thm:track-coverage-converse}
Fix $\varepsilon\in(0,1)$ and $b>0$, and define
$
\overline K_\varepsilon(b)
=
\max\!\left\{
\ell\in\{0,\ldots,|\mathcal G|\}:
(1-\varepsilon)T_\ell^*(\mu)\le b
\right\}.
$
Then, with $B_\delta=b\log(1/\delta)$,
\begin{equation*}
\sup_{\pi\in\Pi_\delta(B_\delta)}
\mathbb P_{\mu,\pi}\left(
|\widehat{\mathcal S}_\pi(B_\delta)|
> \overline K_\varepsilon(b)
\right)
\longrightarrow 0
\quad\text{as }\delta\downarrow0.
\end{equation*}
\end{theorem}
Theorem~\ref{thm:track-coverage-converse} shows that no uniformly reliable policy can asymptotically certify more than $\overline K_\varepsilon(b)$ arms with nonvanishing probability. Together with Theorem~\ref{thm:track-coverage-achievable}, this identifies $T_\ell^*(\mu)$ as the first-order complexity governing achievable coverage: \textsc{BMA-GAI} attains every coverage level whose complexity lies below the budget, whereas no uniformly reliable policy can exceed the corresponding information-theoretic limit.

Combining the achievability and converse results yields the following first-order optimality guarantee.

\begin{corollary}[First-order optimal coverage]
\label{cor:track-coverage-optimality}
Under the Gaussian model, fix $b>0$ with
$b\notin\{T_\ell^*(\mu):1\le\ell\le|\mathcal G|\}$, and define
$
K(b)
=
\max\left\{
\ell\in\{0,\ldots,|\mathcal G|\}:
T_\ell^*(\mu)\le b
\right\}.
$
Then, for $B_\delta=b\log(1/\delta)$,
\begin{equation*}
\mathbb P_\mu\left(
\widehat{\mathcal S}(B_\delta)\subseteq\mathcal G,
\quad
|\widehat{\mathcal S}(B_\delta)|=K(b)
\right)
\longrightarrow 1
\quad\text{as }\delta\downarrow0.
\end{equation*}
Moreover,
\begin{equation*}
\begin{aligned}
\lim_{\delta\downarrow0}
\mathbb E_\mu[|\widehat{\mathcal S}(B_\delta)|]
=
\lim_{\delta\downarrow0}
\sup_{\pi\in\Pi_\delta(B_\delta)}
\mathbb E_{\mu,\pi}[|\widehat{\mathcal S}_\pi(B_\delta)|]
&=
K(b).
\end{aligned}
\end{equation*}
\end{corollary}

Thus, $K(b)$ is the maximum reliable coverage permitted by the problem complexity at normalized budget $b$, and \textsc{BMA-GAI} asymptotically attains it; that is, the number of good arms certified by \textsc{BMA-GAI} matches the information-theoretic limit of the instance to first order. Critical budgets $b=T_\ell^*(\mu)$ are excluded to avoid boundary cases in which lower-order terms determine whether the $\ell$th arm can be certified.

\section{Experiments}
\label{sec:experiments}

We evaluate \textsc{BMA-GAI} on controlled synthetic instances and on a verification task derived from GSM8K~\citep{cobbe2021training}. The synthetic experiments examine reliable coverage under a fixed budget and its dependence on the intrinsic certification complexity, and compare the results with the complexity predicted by our theory. The GSM8K experiment evaluates whether \textsc{BMA-GAI} can effectively allocate a limited verification budget across cheap, medium-cost, and expensive verifiers. 

\paragraph{Baselines.}
We compare \textsc{BMA-GAI} with two baselines. Uniform spreads the budget evenly across active arm--attribute pairs and applies the same certification rule as \textsc{BMA-GAI}, thereby measuring the benefit of adaptive allocation. \textsc{CB-GAI} is a cost-blind variant of good-arm identification that uses the same confidence intervals and threshold-certification rule but ignores verifier costs when selecting arm--attribute pairs, thereby isolating the benefit of cost-aware allocation.

\subsection{Synthetic experiments}

\begin{figure}[t]
    \centering
    \setlength{\abovecaptionskip}{0cm}
    \setlength{\belowcaptionskip}{0cm}
    \includegraphics[width=\linewidth]{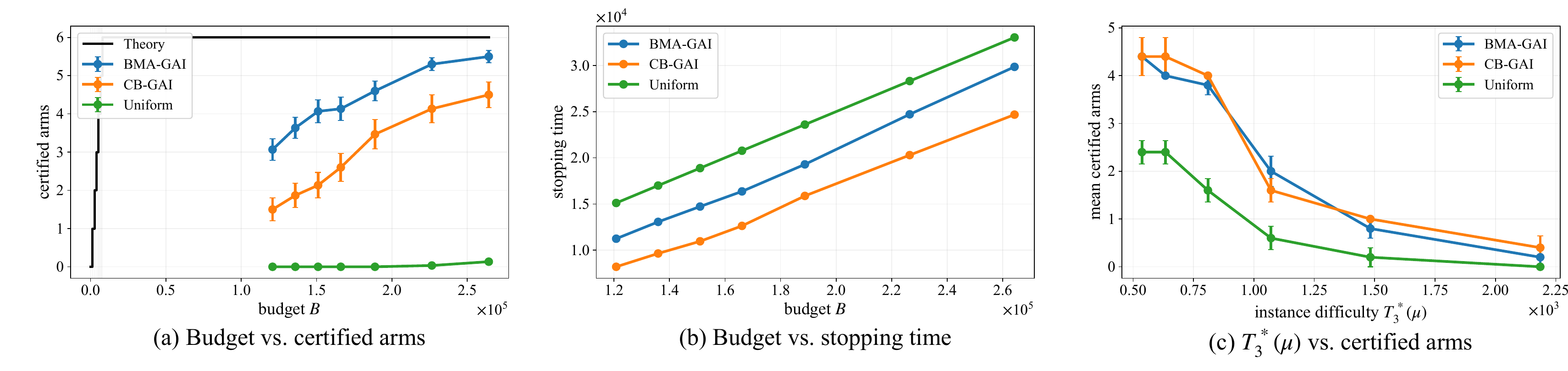}
    \caption{Synthetic experiments. (a) Reliable coverage compared with the characteristic-complexity staircase; (b) the corresponding stopping time; (c) coverage as the characteristic difficulty $T_3^*(\mu)$ varies across problem instances. Error bars indicate standard errors.}
    \label{fig:synthetic-main}
\end{figure}

\paragraph{Experimental setup.}
All synthetic experiments use Gaussian observations
$Y_{q,m}\sim\mathcal N(\mu_{q,m},1)$ with thresholds $\xi_m=0.5$, $Q=12$ arms, and $M=4$ attributes. Figures~\ref{fig:synthetic-main}(a,b) use a heterogeneous-cost instance with $|\mathcal G|=6$ and costs $c=(1,2,4,25)$. The instance is constructed so that some bad arms appear plausible under the cheap attributes and can be ruled out only with evidence from the expensive attribute, whereas good arms can be certified through more cost-efficient attributes. We fix $\delta=0.02$ and vary the budget $B=b\log(1/\delta)$ over seven levels, with normalized budgets $b\in\{16,18,20,22,25,30,35\}\times T_{|\mathcal G|}^*(\mu)$. All methods use the same anytime certification radius and differ only in their allocation rules. Figure~\ref{fig:synthetic-main}(c) uses a second family of instances with costs $c=(1,2,4,8)$, in which the good-arm gaps are scaled to produce different values of $T_3^*(\mu)$; for this experiment, we fix $B=90000$ and $\delta=0.02$. We report the mean number of certified good arms and the mean stopping time, together with standard errors, over 30 independent repetitions.

\paragraph{Results.}
Figure~\ref{fig:synthetic-main} shows a consistent advantage of \textsc{BMA-GAI} over both baselines. As shown in Figure~\ref{fig:synthetic-main}(a), \textsc{BMA-GAI} achieves substantially higher reliable coverage across a wide range of budgets, whereas the coverage of Uniform grows slowly because it wastes observations on arm--attribute pairs that are not critical for certification. Although \textsc{CB-GAI} benefits from adaptive sampling, it ignores attribute-dependent costs and may therefore spend excessive budget on arms that look promising after cheap observations but ultimately require expensive attributes to be certified or ruled out. By jointly accounting for estimated gaps and verifier costs, \textsc{BMA-GAI} prioritizes arms whose remaining certificates are both statistically plausible and cheap to complete, and thus uses the budget more effectively. The staircase pattern in Figure~\ref{fig:synthetic-main}(a) also agrees with the characteristic-complexity thresholds: an additional arm becomes reliably certifiable only once the budget exceeds its corresponding complexity level. Figure~\ref{fig:synthetic-main}(b) further confirms that the gain stems from more efficient sampling rather than a weaker stopping criterion, since all methods use the same anytime certificate. Finally, Figure~\ref{fig:synthetic-main}(c) shows that the coverage of all methods decreases as $T_3^*(\mu)$ increases, supporting its role as an intrinsic measure of instance difficulty; nevertheless, the persistent advantage of \textsc{BMA-GAI} indicates that cost-aware allocation remains beneficial across difficulty levels. Additional experimental results on budget allocation and the interaction between verification cost and statistical difficulty are provided in Appendix~\ref{app:additional-experiments}.

\subsection{Real-world experiments}

\begin{figure}[t]
    \centering
    \setlength{\abovecaptionskip}{0cm}
    \setlength{\belowcaptionskip}{0cm}
    \includegraphics[width=\linewidth]{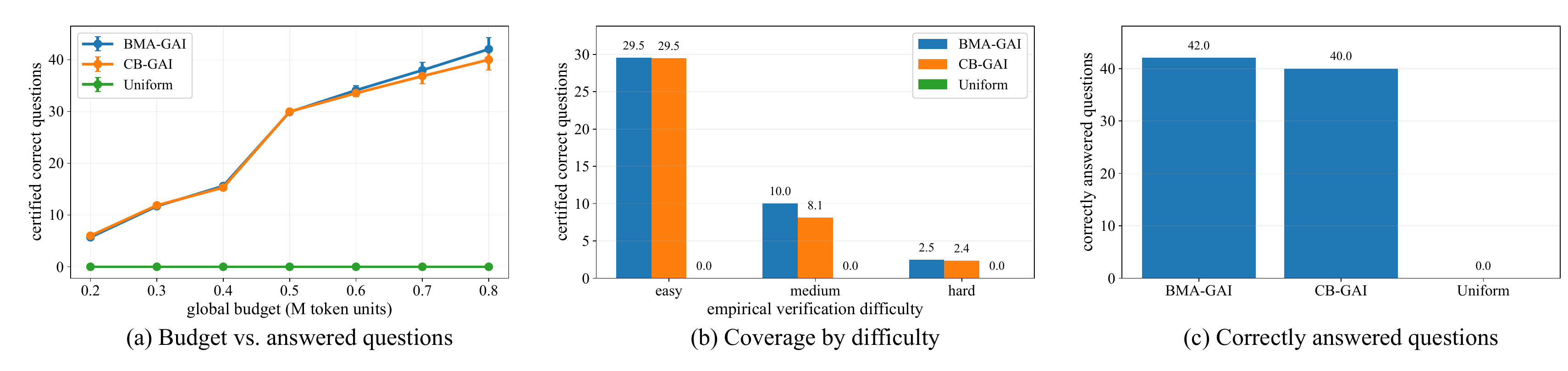}
    \caption{GSM8K experiments. (a) Reliable coverage as the global budget varies; (b) certified correct answers stratified by empirical verification difficulty; (c) final number of correctly answered questions at the largest budget. Error bars indicate standard errors.}
    \label{fig:real-world-gsm8k}
\end{figure}

\paragraph{Experimental setup.}
We construct a real-world instance from 100 GSM8K test questions. For each question, we generate one fixed candidate solution with Qwen3-8B and treat each question--candidate pair as an arm. Each arm has three verifier attributes: the cheap verifier checks whether the solution contains a parseable final answer; the medium verifier solves the question once with Qwen3-8B and compares its answer with the candidate answer; and the expert verifier repeats this procedure five times and compares the majority answer with the candidate answer. Each verifier call returns a binary accept/reject outcome together with its token cost, consistent with our multi-attribute bandit model; during each run, verifier observations are replayed by sampling with replacement from precomputed outcomes. We set the acceptance threshold to $\xi_m=0.5$ for all three attributes. The GSM8K ground-truth answers are never accessible to any algorithm during verification; they are used only after each run to assess whether the certified candidates are correct.

% We vary the global budget from 0.2M to 0.8M token units and evaluate each budget setting over 30 independent runs.

\paragraph{Results.}
Figure~\ref{fig:real-world-gsm8k} demonstrates the advantage of adaptive and cost-aware verification on GSM8K. As shown in Figure~\ref{fig:real-world-gsm8k}(a), Uniform certifies almost no answers under limited budgets: by spreading verification effort evenly across questions and verifiers, it leaves most questions with insufficient evidence to meet the certification criterion. In contrast, \textsc{BMA-GAI} and \textsc{CB-GAI} adaptively concentrate verification on questions that are more likely to be certified. Figure~\ref{fig:real-world-gsm8k}(b) further shows that the adaptive methods certify easy and medium questions first before spending heavily on hard ones, which improves budget efficiency by avoiding excessive effort on questions that remain difficult to resolve. The gap between \textsc{BMA-GAI} and \textsc{CB-GAI} stems mainly from cost awareness: \textsc{BMA-GAI} explicitly accounts for heterogeneous verifier costs and therefore favors verifier calls that yield certification evidence more cost-efficiently. Finally, because certification is performed without access to the ground truth and relies on imperfect verifiers, a certified answer is not necessarily correct. Figure~\ref{fig:real-world-gsm8k}(c) therefore compares the certified candidates with the gold answers after each run to measure the actual useful coverage, on which \textsc{BMA-GAI} again achieves the best overall performance.

\section{Conclusion and Future Work}
\label{sec:conclusion}

This paper studied multi-attribute good-arm identification under a global verification budget, motivated by the allocation of test-time verification computation across LLM-generated candidate answers. We proposed \textsc{BMA-GAI}, which jointly selects candidate answers and verification attributes via cost-aware adaptive allocation while providing anytime-valid certification under a hard budget. Our analysis identified the aggregate certification complexity $T_\ell^*(\mu)$ as the fundamental quantity governing reliable coverage. We established an achievable coverage guarantee together with a matching information-theoretic converse, showing that \textsc{BMA-GAI} attains the optimal coverage to first order away from critical budget thresholds. Experiments on synthetic instances and GSM8K further demonstrated that adaptive, cost-aware verification uses a shared budget more efficiently than uniform or cost-blind allocation.

Several directions remain open. An important one is to jointly allocate computation between candidate generation and verification. It would also be interesting to study richer verification models with unknown or time-varying evaluation costs, and to derive sharper finite-budget guarantees.

% \subsection*{AI use statement}

% We used generative AI tools to assist with drafting and revising the manuscript, checking technical consistency, and developing experimental code. In the GSM8K experiments, Qwen3-8B was used to generate the candidate solutions and the solution samples underlying the verification procedures. The authors take full responsibility for the problem formulation, theoretical claims, experimental methodology, and reported results, including all text, code, and figures prepared with AI assistance.

% \subsection*{Reproducibility statement}

% The problem formulation and the \textsc{BMA-GAI} algorithm are presented in Sections~\ref{sec:prelim} and~\ref{sec:track-algorithm}, and proofs of the main theoretical results are given in Appendix~\ref{app:track-proofs}. Section~\ref{sec:experiments} describes the synthetic Gaussian experiments and the GSM8K-derived verification task, together with the baselines and evaluation metrics. The GSM8K experiments use one fixed candidate answer per question and replay observations by sampling with replacement from precomputed verifier outcomes with recorded token costs. To facilitate reproduction, we will publicly release the code and materials upon acceptance, including synthetic instance generators, implementations of the algorithm and baselines, GSM8K question identifiers, candidate solutions, verifier outcomes and costs, experimental configurations and random seeds, raw trial outputs, result summaries, and plotting scripts.

\bibliographystyle{plainnat}
\bibliography{ref}

\newpage

\appendix
\section{Additional Experiments}
\label{app:additional-experiments}

We present two additional synthetic diagnostics that examine how the verification budget is distributed across arms and how cost heterogeneity interacts with statistical difficulty. Both use the same three methods as Section~\ref{sec:experiments}. The arm categories shown in the figures are used only for analysis and are never revealed to any algorithm.

\begin{figure}[h]
    \centering
    \setlength{\abovecaptionskip}{0cm}
    \setlength{\belowcaptionskip}{0cm}
    \includegraphics[width=\linewidth]{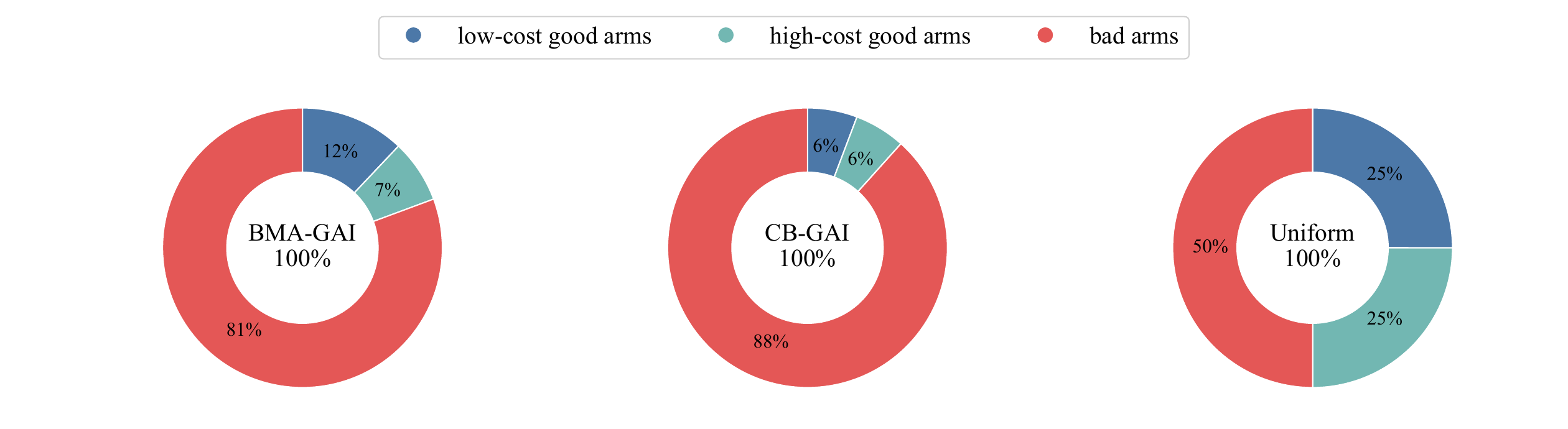}
    \caption{Final budget allocation on the expensive-decoy synthetic instance. Each circular panel corresponds to one algorithm after the budget is exhausted, and its three sectors show the budget shares spent on low-certification-cost good arms, high-certification-cost good arms, and bad arms.}
    \label{fig:app-allocation-process}
\end{figure}

\paragraph{Final budget allocation by arm type.}
Figure~\ref{fig:app-allocation-process} shows the final budget allocation across three arm types on a stress-test instance with $Q=12$ arms, $M=4$ attributes, costs $c=(1,2,4,25)$, and $|\mathcal G|=6$. The bad arms appear statistically promising on the inexpensive attributes and fail only on the costly one. The good arms are split into low- and high-certification-cost groups according to their characteristic costs $h_q$. The total budget is set to $12T_{|\mathcal G|}^*(\mu)\log(1/\delta)$ with $\delta=0.02$, and each panel averages results over eight independent repetitions.

The resulting allocation patterns reveal clear differences among the methods. Uniform spreads its budget nearly evenly across good and bad arms and therefore cannot accumulate enough evidence to certify good arms in this budget regime. Both adaptive methods instead concentrate their budgets on statistically promising arms. However, \textsc{CB-GAI} spends a larger fraction of its budget on the bad decoy arms because it ignores verifier costs when selecting targets. In contrast, \textsc{BMA-GAI} devotes more budget to low-certification-cost good arms, reflecting its preference for arms whose certificates are both statistically plausible and cheap to complete.

\begin{figure}[h]
    \centering
    \setlength{\abovecaptionskip}{0cm}
    \setlength{\belowcaptionskip}{0cm}
    \includegraphics[width=0.9\linewidth]{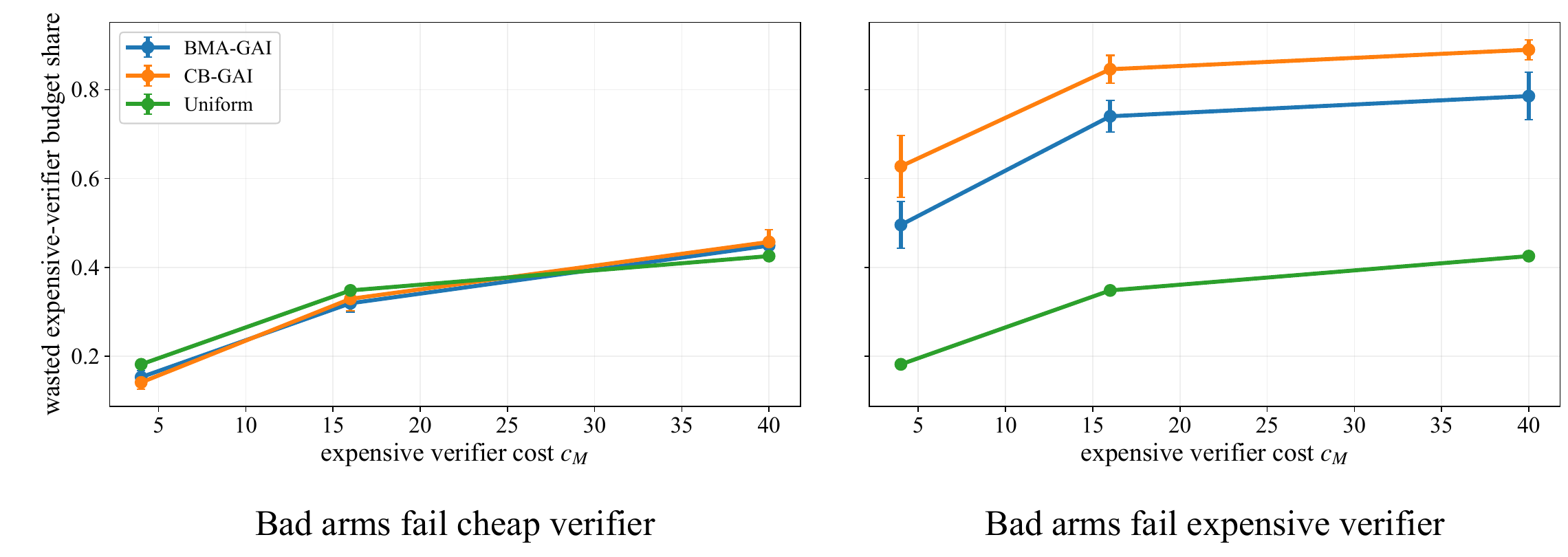}
    \caption{Cost--difficulty interaction. The vertical axis shows the cost of expensive-verifier calls on bad arms as a fraction of the total budget; error bars denote standard errors over eight repetitions. The left and right panels differ in which attribute reveals that a bad arm is bad.}
    \label{fig:app-cost-difficulty-interaction}
\end{figure}

\paragraph{Interaction between cost and difficulty.}
This experiment examines whether the benefit of cost-aware allocation depends on \emph{where} the evidence for rejecting a bad arm is concentrated. Both instance families contain $Q=12$ arms and $M=4$ attributes, with six good arms sharing the same mean vectors and verifier costs $(1,2,4,c_M)$, where $c_M\in\{4,16,40\}$. In both families, every bad arm falls below the threshold by the same small gap. In the left panel, this gap occurs on the cheapest attribute, so bad arms can be rejected with inexpensive evidence. In the right panel, the gap occurs only on the fourth attribute: the cheaper attributes make the bad arms look promising, and rejecting them requires costly observations. For each instance, we set $B=12T_{|\mathcal G|}^*(\mu)\log(1/\delta)$ with $\delta=0.02$ and average the results over eight independent repetitions. The reported quantity is the fraction of the total budget $B$ spent on applying the fourth verifier to truly bad arms.

When bad arms can be rejected with the cheapest verifier, \textsc{BMA-GAI} and \textsc{CB-GAI} spend similar fractions of the budget on expensive verification of bad arms, even as the cost of the fourth verifier increases. In contrast, when bad arms can be ruled out only through the expensive verifier, \textsc{CB-GAI} consistently spends a larger fraction of the budget on costly bad-arm observations than \textsc{BMA-GAI}, and the gap widens as $c_M$ increases. This more efficient allocation also translates into substantially more certified good arms for \textsc{BMA-GAI} in the high-cost regime. Although Uniform may spend a smaller fraction on expensive verification of bad arms, it fails to certify good arms under the same budget, showing that a small fraction alone does not indicate effective allocation. Together, the two panels show that cost awareness is most beneficial when costly verification coincides with the main source of identification difficulty.

\section{Proofs for \textsc{BMA-GAI}}
\label{app:track-proofs}

Throughout the appendix, write
\[
    D=QM,\quad
    c_{\min}=\min_m c_m,\quad
    c_{\max}=\max_m c_m.
\]
In both coverage theorems, $b>0$ is fixed and
$B_\delta=b\log(1/\delta)$.
Define the certificate radius
\begin{equation*}
    r(n,\delta)
    =\sqrt{\frac{2}{n}\log\frac{4Dn^2}{\delta}},
    \quad n\ge1.
\end{equation*}
For any execution, let
$\sigma_{q,m}(n)=\inf\{t:N_{q,m}(t)=n\}$, with $\inf\varnothing=\infty$.
On $\{\sigma_{q,m}(n)<\infty\}$, write
$\bar X_{q,m}(n)$ for the empirical mean of its first $n$ observations
of pair $(q,m)$. Unlike $\widehat\mu_{q,m}(t)$, this notation is indexed
by the local sample count. The sub-Gaussian proofs never require an
empirical mean at an unattained local count.

\subsection{Proof of Proposition~\ref{thm:track-safety}}

\begin{proof}
Fix $B>0$. For a pair $(q,m)$, let
$I_s^{q,m}=\mathbf 1\{(q_s,m_s)=(q,m)\}$ and
\[
    W_{q,m}(t)
    =\sum_{s=1}^t I_s^{q,m}(Y_s-\mu_{q,m}).
\]
For every fixed $\lambda\in\mathbb R$, the conditional moment assumption
implies that
\[
    Z_{q,m}^{\lambda}(t)
    =\exp\!\left(
        \lambda W_{q,m}(t)-\frac{\lambda^2}{2}N_{q,m}(t)
    \right)
\]
is a nonnegative supermartingale starting at one; after termination,
we keep this process constant. Optional stopping at
$\sigma_{q,m}(n)\wedge T$, followed by Fatou's lemma as $T\to\infty$,
gives
\[
    \mathbb E_\mu\!\left[
        \mathbf 1\{\sigma_{q,m}(n)<\infty\}
        \exp\!\left(
            \lambda n(\bar X_{q,m}(n)-\mu_{q,m})
            -\frac{\lambda^2n}{2}
        \right)
    \right]\le1.
\]
Optimizing the resulting exponential bound at $\lambda=x$ or $-x$
yields, for either choice of sign,
\begin{equation}
    \mathbb P_\mu\!\left(
        \sigma_{q,m}(n)<\infty,\;
        \pm(\bar X_{q,m}(n)-\mu_{q,m})\ge x
    \right)
    \le e^{-nx^2/2},\quad x>0.
    \label{eq:app-track-local-concentration}
\end{equation}
This argument applies to any admissible adaptive execution, not only to
the finite-budget algorithm.

Taking $x=r(n,\delta)$ and combining the two tails gives
\begin{equation*}
\begin{aligned}
&\mathbb P_\mu\!\left(
    \sigma_{q,m}(n)<\infty,\;
    |\bar X_{q,m}(n)-\mu_{q,m}|>r(n,\delta)
\right)
\\
&\quad\le 2e^{-nr^2(n,\delta)/2}
=\frac{\delta}{2Dn^2}.
\end{aligned}
\end{equation*}
Note that the event includes attainment of the local count; the
probability is not conditioned on attainment. A union bound over all
pairs and counts therefore gives
\begin{equation*}
\begin{aligned}
&\mathbb P_\mu\!\left(
    \exists(q,m),n\ge1:\ \sigma_{q,m}(n)<\infty,\;
    |\bar X_{q,m}(n)-\mu_{q,m}|>r(n,\delta)
\right)
\\
&\quad\le D\sum_{n=1}^{\infty}\frac{\delta}{2Dn^2}
=\frac{\pi^2}{12}\delta\le\delta.
\end{aligned}
\end{equation*}
On the complementary event, every attained raw lower bound is at most
$\mu_{q,m}$. Since $\mu_{q,m}\ge0$, the running maximum defining
$L_{q,m}(t)$, including its initial value zero, also satisfies
$L_{q,m}(t)\le\mu_{q,m}$. Thus, certification of $q$ implies
$\mu_{q,m}\ge\xi_m$ for every $m$, and hence $q\in\mathcal G$.
This holds simultaneously at all times up to termination and proves
Eq.~(\ref{eq:track-safety}).

Every sampling branch selects an attribute from the current set
$\mathcal M_t=\{m:C_t+c_m\le B\}$. In the ordinary branch, the
operational-index convention in Eq.~(\ref{eq:track-cost-interval})
ensures that the selected target has a nonempty affordable unresolved
set. Consequently, every executed observation satisfies
$C_{t+1}=C_t+c_{m_{t+1}}\le B$. Induction gives $C_{\TS}\le B$ almost
surely, and positivity of the costs gives
$\TS\le\lfloor B/c_{\min}\rfloor$.

Storing counts, empirical means, and certificates uses $O(QM)$ memory.
Direct scans for exploration, cost indices, selection, and certification
use $O(QM)$ operations per observation, excluding verifier execution.
\end{proof}

\subsection{Proof of Theorem~\ref{thm:track-coverage-achievable}}

Throughout this subsection, we assume only the conditional
$1$-sub-Gaussian observation model and a fixed nonboundary mean matrix,
i.e., $\mu_{q,m}\neq\xi_m$ for all $(q,m)$. Neither Gaussianity nor
independence across observations is required.

\paragraph{Reference process.}
Consider the action rule obtained from Algorithm~\ref{alg:bgai-track} by
setting $\mathcal M_t=[M]$ at every decision, with all other rules
unchanged. We refer to an execution of this rule as a budget-free
reference process. For every active arm, $\mathcal U_q(t)\ne\varnothing$, so its
operational cost indices agree with $H(u_q(t))$ and $H(\ell_q(t))$.
Define
\[
    \tau_\ell(\delta)=\inf\{t:|\mathcal S_t|\ge\ell\},\quad
    \mathcal C_\ell(\delta)=C_{\tau_\ell(\delta)},
\]
with $\mathcal C_\ell(\delta)=\infty$ when $\tau_\ell(\delta)=\infty$.
All probability estimates below hold uniformly over the conditional
observation laws satisfying Eq.~(\ref{eq:subgaussian-verifier}) with
these means. In particular, they do not require different confidence
levels to share a common observation stream.

\begin{lemma}[Sparse-exploration count]
\label{lem:track-exploration-count}
Let $F(t)$ count the forced-exploration observations among the first $t$
observations of the reference process. Then
\begin{equation}
    F(t)\le D\left\lceil\frac{\sqrt t}{D}\right\rceil
    \le\sqrt t+D.
    \label{eq:app-track-exploration-count}
\end{equation}
For $t\ge t_D:=16(D+1)^2$, every arm still active after observation $t$
satisfies
\begin{equation*}
    N_{q,m}(t)
    \ge\left\lfloor\frac{\sqrt{t/2}}{D}\right\rfloor
    \ge\frac{\sqrt t}{3D},\quad m\in[M].
\end{equation*}
\end{lemma}

\begin{proof}
When a pair is selected by forced exploration at observation $s\le t$,
its pre-draw count is smaller than $\sqrt{s}/D\le\sqrt t/D$.
Its count increases after every such selection, so it can be selected
this way at most $\lceil\sqrt t/D\rceil$ times. Summing over all $D$
pairs proves Eq.~(\ref{eq:app-track-exploration-count}).

Suppose an arm remains active at time $t$ but one of its pairs has count
smaller than $\lfloor\sqrt{t/2}/D\rfloor$. The arm was active at all
earlier times because certification is irreversible. At every observation
$s$ in the last half of the first $t$ observations, that pair's pre-draw
count is below $\sqrt{s}/D$. Each such observation would therefore be
forced exploration, implying $F(t)\ge\lfloor t/2\rfloor$.
For $t\ge t_D$, this contradicts
$\lfloor t/2\rfloor>\sqrt t+D$. The final floor bound follows from
$t\ge16(D+1)^2$.
\end{proof}

\begin{lemma}[High-probability stabilization]
\label{lem:track-subg-stabilization}
There exists a deterministic $t_*(\mu)<\infty$ such that, for every
integer $u\ge2$, there is an event $\mathcal A_u$ with
\begin{equation*}
    \mathbb P_\mu(\mathcal A_u^c)\le\frac{D}{(u-1)^2},
\end{equation*}
on which, at every non-exploration decision with
$t\ge\max\{u,t_*(\mu)\}$ at which some good arm is still active,
$\min_{q\in\mathcal Q_t}H_q^+(t)<\infty$ and
\begin{equation*}
    \mathcal P_t
    =\arg\min_{q\in\mathcal Q_t\cap\mathcal G}h_q.
\end{equation*}
\end{lemma}

\begin{proof}
Let $\mathcal A_u$ be the event that
\begin{equation*}
    |\bar X_{q,m}(n)-\mu_{q,m}|
    \le\sqrt{\frac{8\log(t+1)}{n}}
\end{equation*}
for every integer $t\ge u$, every pair, and every $1\le n\le t$
satisfying $\sigma_{q,m}(n)\le t$. By
Eq.~(\ref{eq:app-track-local-concentration}),
\begin{equation*}
\begin{aligned}
    \mathbb P_\mu(\mathcal A_u^c)
    \le2D\sum_{t=u}^{\infty}\frac{t}{(t+1)^4}
    \le2D\sum_{t=u}^{\infty}t^{-3}
    \le\frac{D}{(u-1)^2}.
\end{aligned}
\end{equation*}
On $\mathcal A_u$, every active pair with $t\ge\max\{u,t_D\}$ has
allocation radius at most
\begin{equation*}
    e(t)=\sqrt{\frac{24D\log(t+1)}{\sqrt t}}
    \longrightarrow0,
\end{equation*}
by Lemma~\ref{lem:track-exploration-count}. Its allocation interval
contains the mean, and both endpoints lie within $2e(t)$ of the mean.

Since all threshold gaps are nonzero, eventually every active good arm
has all lower endpoints strictly above the thresholds, whereas every
active bad arm has an upper endpoint strictly below its threshold on at
least one attribute. Thus bad arms have $H_q^-=H_q^+=\infty$, and good
arms satisfy $H_q^-\le h_q\le H_q^+$, with both endpoints converging
uniformly to $h_q$.

If there are distinct good-arm complexities, let $\rho>0$ be their
smallest positive separation. Choose a deterministic time large enough
such that both cost endpoints of every good arm differ from $h_q$ by less
than $\rho/3$. An arm with complexity larger than the smallest remaining
complexity then has its optimistic cost strictly above a conservative
cost of a smallest-complexity arm, and is excluded from $\mathcal P_t$.
Every arm tied at the smallest remaining complexity $h_{\min}$ remains
in $\mathcal P_t$, since
$H_q^-\le h_{\min}\le\min_{i\in\mathcal Q_t}H_i^+$.
If all good-arm complexities are equal, only this latter argument is
needed. These bounds apply to every active subset. Taking $t_*(\mu)$
to include these deterministic requirements proves the claim. Fixed-index
tie-breaking selects the first remaining good arm in the order $(h_q,q)$.
\end{proof}

For a good pair, define the first crossing among its attained counts:
\begin{equation*}
    n_{q,m}^{\dagger}(\delta)
    =\inf\left\{n\ge1:
        \sigma_{q,m}(n)<\infty,\;
        \bar X_{q,m}(n)-r(n,\delta)\ge\xi_m
    \right\}.
\end{equation*}
The value is $\infty$ if the set is empty.

\begin{lemma}[High-probability one-attribute certification]
\label{lem:track-subg-scalar}
For each good pair and fixed $\eta>0$,
\begin{equation*}
    \mathbb P_\mu\!\left(
        n_{q,m}^{\dagger}(\delta)>
        \frac{2(1+\eta)}{\Delta_{q,m}^2}\log(1/\delta)
    \right)\longrightarrow0.
\end{equation*}
Consequently,
\begin{equation}
    \mathbb P_\mu\!\left(
        n_{q,m}^{\dagger}(\delta)\le
        \frac{2(1+\eta)}{\Delta_{q,m}^2}\log(1/\delta),
        \quad\forall q\in\mathcal G,\ m\in[M]
    \right)\longrightarrow1.
    \label{eq:app-track-subg-scalar-simultaneous}
\end{equation}
\end{lemma}

\begin{proof}
Fix a good pair and abbreviate its mean, threshold, and gap by
$\nu$, $\xi$, and $\Delta=\nu-\xi>0$.
Set
\[
    n_\delta=
    \left\lfloor
        \frac{2(1+\eta)}{\Delta^2}\log(1/\delta)
    \right\rfloor.
\]
For sufficiently small $\delta$, $n_\delta\ge1$ and
\[
    r(n_\delta,\delta)\longrightarrow
    \frac{\Delta}{\sqrt{1+\eta}}<\Delta.
\]
Choose a fixed $a_\eta>0$ such that, for all sufficiently small $\delta$,
$r(n_\delta,\delta)\le\Delta-a_\eta$.

On the event $\{n^\dagger(\delta)>n_\delta\}$, the local count
$n_\delta$ must be attained. Otherwise, this attribute never crosses its
threshold, its arm is never certified, and the reference process keeps
that arm active forever; the count lower bound in
Lemma~\ref{lem:track-exploration-count} then contradicts nonattainment
of $n_\delta$. At the attained count, absence of a crossing gives
\[
    \bar X(n_\delta)-\nu
    <r(n_\delta,\delta)-\Delta\le-a_\eta.
\]
Therefore Eq.~(\ref{eq:app-track-local-concentration}) implies
\[
    \mathbb P_\mu(n^\dagger(\delta)>n_\delta)
    \le e^{-n_\delta a_\eta^2/2}\longrightarrow0.
\]
Since $n^\dagger$ is integer-valued or infinite, the event on the left
is exactly
$\{n^\dagger>2(1+\eta)\log(1/\delta)/\Delta^2\}$. A union bound over the
finitely many good pairs proves the simultaneous statement.
\end{proof}

\begin{proposition}[Reference-process completion upper bound]
\label{prop:track-reference-upper}
For every fixed $\alpha\in(0,1)$, under the conditional sub-Gaussian model,
\begin{equation}
    \mathbb P_\mu\!\left(
        \mathcal C_\ell(\delta)
        \le(1+\alpha)T_\ell^*(\mu)\log(1/\delta),
        \quad\forall\ell\in\{1,\ldots,|\mathcal G|\}
    \right)\longrightarrow1.
    \label{eq:app-track-reference-upper}
\end{equation}
The convergence is uniform over admissible conditional observation laws
with the fixed mean matrix.
\end{proposition}

\begin{proof}
If $|\mathcal G|=0$, the statement is vacuous. Otherwise set
$\eta=\alpha/4$ and
$u_\delta=\max\{2,\lceil[\log(1/\delta)]^{1/3}\rceil\}$.
Order the good arms by $(h_q,q)$ as
$q_{(1)},\ldots,q_{(|\mathcal G|)}$, and set
$S_\ell^*=\{q_{(1)},\ldots,q_{(\ell)}\}$.
Then $\sum_{q\in S_\ell^*}h_q=T_\ell^*(\mu)$ by
Eq.~(\ref{eq:track-characteristic-closed}).
Let $\mathcal B_\delta$ be the event in
Eq.~(\ref{eq:app-track-subg-scalar-simultaneous}). The preceding lemmas give
\begin{equation*}
    \mathbb P_\mu(\mathcal A_{u_\delta}\cap\mathcal B_\delta)
    \longrightarrow1.
\end{equation*}
For sufficiently small $\delta$, $u_\delta\ge t_*(\mu)$.
On this intersection, define
\[
    A_{\ell,\delta}^{\dagger}
    =\sum_{q\in S_\ell^*}\sum_{m=1}^M
        c_m n_{q,m}^{\dagger}(\delta).
\]
Then
\begin{equation*}
    A_{\ell,\delta}^{\dagger}
    \le(1+\eta)T_\ell^*(\mu)\log(1/\delta),
    \quad 1\le\ell\le|\mathcal G|.
\end{equation*}

We charge the first $u_\delta$ observations at cost at most
$c_{\max}u_\delta$. Subsequently, while fewer than $\ell$ outputs have
occurred, every ordinary target belongs to $S_\ell^*$. Indeed, stable
selection of a good arm with rank larger than $\ell$ would require all
of the first $\ell$ good arms to have already been certified. This
argument remains valid even if erroneous certifications occur earlier,
because arms are removed only through certification.

An ordinary draw samples an unresolved attribute, whose pre-draw count
must be smaller than $n_{q,m}^{\dagger}(\delta)$. Once this count has
been reached, the running certificate stays above threshold. The total
post-$u_\delta$ ordinary cost is therefore at most
$A_{\ell,\delta}^{\dagger}$. Stabilization also excludes the fallback
branch while a good arm remains active. All other draws are forced
exploration. Thus every prefix up to and including the observation
producing the $\ell$th output satisfies
\begin{equation*}
    C(t)\le A_{\ell,\delta}^{\dagger}
        +c_{\max}(u_\delta+D)+c_{\max}\sqrt t.
\end{equation*}
Writing
\[
    W_{\ell,\delta}
    =A_{\ell,\delta}^{\dagger}+c_{\max}(u_\delta+D),
    \quad \kappa=\frac{c_{\max}}{\sqrt{c_{\min}}},
\]
and using $C(t)\ge c_{\min}t$, we obtain
$C(t)\le W_{\ell,\delta}+\kappa\sqrt{C(t)}$. Solving this inequality gives
\begin{equation}
\begin{aligned}
    C(t)
    &\le W_{\ell,\delta}+\frac{\kappa^2}{2}
       +\frac{\kappa}{2}\sqrt{\kappa^2+4W_{\ell,\delta}}\\
    &\le W_{\ell,\delta}+\kappa^2
       +\kappa\sqrt{W_{\ell,\delta}}.
\end{aligned}
\label{eq:app-track-subg-quadratic-cost}
\end{equation}
The $\ell$th output must occur on the event under consideration;
otherwise, a good arm remains active forever, so the reference process
keeps sampling and $C(t)\ge c_{\min}t\to\infty$, contradicting
Eq.~(\ref{eq:app-track-subg-quadratic-cost}).

Because $u_\delta=o(\log(1/\delta))$, the deterministic bounds on this event
imply, simultaneously for the finitely many $\ell$,
\[
    W_{\ell,\delta}
    \le(1+\eta)T_\ell^*(\mu)\log(1/\delta)
       +o(\log(1/\delta)),
    \quad
    \kappa^2+\kappa\sqrt{W_{\ell,\delta}}
    =o(\log(1/\delta)).
\]
Since $\eta<\alpha$ and every $T_\ell^*(\mu)>0$ for $\ell\ge1$,
Eq.~(\ref{eq:app-track-reference-upper}) follows.
The bounds on $\mathbb P_\mu(\mathcal A_{u_\delta}^c)$ and
$\mathbb P_\mu(\mathcal B_\delta^c)$ depend only on $\delta$ and the
fixed instance parameters, not on the particular conditional observation
laws, which proves the asserted uniformity.
\end{proof}

\begin{lemma}[Budget coupling]
\label{lem:track-budget-coupling}
For each fixed $(B,\delta)$, a finite-budget execution and an admissible
reference execution can be coupled to agree on every reference prefix
whose cumulative cost is at most $B-c_{\max}$. In particular,
\begin{equation*}
    \mathcal C_\ell(\delta)\le B-c_{\max}
    \quad\Longrightarrow\quad
    |\widehat{\mathcal S}(B)|\ge\ell,
    \quad 1\le\ell\le|\mathcal G|.
\end{equation*}
\end{lemma}

\begin{proof}
If $B<c_{\max}$, the displayed implication is vacuous. Otherwise, start
both executions with the same action--observation history and fixed
tie-breaking rules. Whenever their histories agree and
$C_t\le B-c_{\max}$, every attribute is affordable because
\[
    C_t+c_m\le C_t+c_{\max}\le B,\quad m\in[M].
\]
Thus $\mathcal M_t=[M]$ in the finite-budget execution. Its exploration,
cost-index, target-selection, and attribute-selection rules agree with
those of the reference process. Feeding both executions the same next
observation, drawn from the finite-budget experiment's conditional
feedback law, keeps their counts, means, certificates, and certified
sets in agreement.

We continue synchronously while the cost condition holds. Once it fails,
the reference process can be extended using any conditional observation
law satisfying Eq.~(\ref{eq:subgaussian-verifier}) with the same means;
for example, it may receive the selected pair's mean deterministically.
This is only an auxiliary probabilistic construction, not a change to
the algorithm. It requires neither independent streams nor a common
sample path for different budgets. Proposition~\ref{prop:track-reference-upper}
applies uniformly to such continuations.

Induction proves agreement on all the stated prefixes. If
$\mathcal C_\ell(\delta)\le B-c_{\max}$, the finite-budget execution
has made the same $\ell$ certifications by that point. Subsequent
budget-aware decisions cannot revoke them, proving the implication.
\end{proof}

\begin{proof}[Proof of Theorem~\ref{thm:track-coverage-achievable}]
Fix $b>0$ and $\varepsilon\in(0,1)$, and write
$k=\underline K_\varepsilon(b)$. If $k=0$, reliability alone proves
the claim. Suppose $k\ge1$. By definition,
\[
    d_b:=b-(1+\varepsilon)T_k^*(\mu)>0.
\]
For sufficiently small $\delta$,
$d_b\log(1/\delta)\ge c_{\max}$, so
\begin{equation}
    (1+\varepsilon)T_k^*(\mu)\log(1/\delta)
    \le b\log(1/\delta)-c_{\max}=B_\delta-c_{\max}.
    \label{eq:app-track-normalized-margin}
\end{equation}
Construct the coupled reference process from
Lemma~\ref{lem:track-budget-coupling}. By
Proposition~\ref{prop:track-reference-upper}, with $\alpha=\varepsilon$,
\[
    \mathbb P_\mu\!\left(
        \mathcal C_k(\delta)\le
        (1+\varepsilon)T_k^*(\mu)\log(1/\delta)
    \right)\longrightarrow1.
\]
Its uniformity allows the auxiliary continuation to depend on
$(B_\delta,\delta)$. Combining this estimate with
Eq.~(\ref{eq:app-track-normalized-margin}) and the coupling lemma gives
$\mathbb P_\mu(|\widehat{\mathcal S}(B_\delta)|\ge k)\to1$.
Finally,
\begin{align*}
&\mathbb P_\mu\!\left(
    \widehat{\mathcal S}(B_\delta)\subseteq\mathcal G,
    \ |\widehat{\mathcal S}(B_\delta)|\ge k
\right)\\
&\quad\ge1-\delta-
\mathbb P_\mu\!\left(
    \mathcal C_k(\delta)>
    (1+\varepsilon)T_k^*(\mu)\log(1/\delta)
\right)\longrightarrow1,
\end{align*}
by Proposition~\ref{thm:track-safety}. This proves the theorem.
\end{proof}

\subsection{Proof of Theorem~\ref{thm:track-coverage-converse}}

In this subsection, observations are independent unit-variance Gaussian
streams. They may be realized as independent arrays
$X_{q,m,n}\sim\mathcal N(\mu_{q,m},1)$, revealing $X_{q,m,n}$ at the
$n$th draw of $(q,m)$. This Gaussian construction is not used in the
sub-Gaussian achievability proof.

\paragraph{Variational characterization.}
For nonempty $S\subseteq\mathcal G$, let
$\operatorname{Alt}(S)=\{\lambda\in[0,1]^{Q\times M}:
S\nsubseteq\mathcal G(\lambda)\}$ and define
\begin{equation*}
    \Gamma(S)
    =\max_{w\in\Delta_D}\inf_{\lambda\in\operatorname{Alt}(S)}
    \sum_{q,m}\frac{w_{q,m}}{c_m}
    \frac{(\mu_{q,m}-\lambda_{q,m})^2}{2},
\end{equation*}
where $\Delta_D$ is the probability simplex of cost fractions.

\begin{lemma}[Variational characterization]
\label{lem:track-characteristic}
For every nonempty $S\subseteq\mathcal G$,
\begin{equation}
    \Gamma(S)=\frac{1}{\sum_{q\in S}h_q}.
    \label{eq:app-track-gamma-closed}
\end{equation}
Consequently, for $1\le\ell\le|\mathcal G|$,
\begin{equation*}
    \left[\max_{S\subseteq\mathcal G:\,|S|=\ell}\Gamma(S)\right]^{-1}
    =T_\ell^*(\mu).
\end{equation*}
\end{lemma}

\begin{proof}
Every alternative invalidating $S$ has some $q\in S,m\in[M]$ with
$\lambda_{q,m}<\xi_m$. The information contributed by this coordinate
is at least $w_{q,m}\Delta_{q,m}^2/(2c_m)$. Conversely, changing just
one such coordinate and approaching $\xi_m$ from below attains this
value as an infimum. Since $\xi_m\in(0,1)$, these alternatives remain
in the mean domain. Therefore, for each fixed $w$,
\begin{equation*}
    \inf_{\lambda\in\operatorname{Alt}(S)}
    \sum_{q,m}\frac{w_{q,m}}{c_m}
        \frac{(\mu_{q,m}-\lambda_{q,m})^2}{2}
    =\min_{q\in S,m\in[M]}\frac{w_{q,m}}{h_{q,m}}.
\end{equation*}
If the minimum equals $z$, then $w_{q,m}\ge z h_{q,m}$ on $S$.
Summing gives $1\ge z\sum_{q\in S}h_q$. Equality is achieved by
$w_{q,m}=h_{q,m}/\sum_{i\in S}h_i$ on $S$, with zero weight elsewhere.
This proves Eq.~(\ref{eq:app-track-gamma-closed}). Maximizing over
size-$\ell$ sets selects the $\ell$ smallest good-arm complexities.
The corresponding within-arm cost fractions are $h_{q,m}/h_q$.
\end{proof}

\begin{lemma}[Necessary evidence for returning a good arm]
\label{lem:track-necessary-evidence}
Fix $\varepsilon\in(0,1)$. For all sufficiently small $\delta$,
uniformly over $B>0$, $\pi\in\Pi_\delta(B)$, $q\in\mathcal G$, and
$m\in[M]$,
\begin{equation*}
\begin{aligned}
&\mathbb P_{\mu,\pi}\!\left(
    q\in\widehat{\mathcal S}_\pi(B),\;
    N_{q,m}(\TS)<
    \frac{2(1-\varepsilon)}{\Delta_{q,m}^2}\log(1/\delta)
\right)\\
&\quad\le2\delta^{\varepsilon^2/64}.
\end{aligned}
\end{equation*}
The required upper bound on $\delta$ depends only on the fixed instance
and $\varepsilon$.
\end{lemma}

\begin{proof}
Fix a good pair, put $\nu=\mu_{q,m}$, $\xi=\xi_m$, and
$\Delta=\nu-\xi>0$. Choose
\[
    u=\min\left\{\frac{\xi}{2},\;
        \Delta\left(\sqrt{1+\varepsilon/4}-1\right)\right\}>0.
\]
Then $u<\xi$ and
\begin{equation}
    \frac{(\Delta+u)^2}{\Delta^2}\le1+\frac{\varepsilon}{4}.
    \label{eq:app-track-alt-choice}
\end{equation}
Define an alternative mean matrix $\lambda$ by changing only
$\lambda_{q,m}=\xi-u$. Arm $q$ is bad on this admissible alternative,
so uniform reliability gives
\begin{equation*}
    \mathbb P_{\lambda,\pi}(q\in\widehat{\mathcal S}_\pi(B))\le\delta.
\end{equation*}
Write $d=\Delta+u$ and
$n_\delta=\lfloor2(1-\varepsilon)\log(1/\delta)/\Delta^2\rfloor$.
For small enough $\delta$, $n_\delta\ge1$.
Under the true instance, let
$S_k=\sum_{i=1}^k(X_{q,m,i}-\nu)$, with $S_0=0$. The local
log-likelihood ratio is
\begin{equation*}
    Z_k=dS_k+\frac{d^2k}{2}.
\end{equation*}

The hard budget bounds the stopping time by $\lfloor B/c_{\min}\rfloor$
under both instances. We pad the transcript after stopping to this common
finite horizon. The policy's action, randomization, and output kernels
are identical under both instances, and observations of all other pairs
have identical laws. Thus the full transcript likelihood ratio is
$\exp(Z_N)$, where $N=N_{q,m}(\TS)$.

Let $A=\{q\in\widehat{\mathcal S}_\pi(B),\ N\le n_\delta\}$ and
$x_\delta=(1-\varepsilon/2)\log(1/\delta)$. Splitting at $Z_N=x_\delta$
gives
\begin{equation}
\begin{aligned}
    \mathbb P_{\mu,\pi}(A)
    &\le e^{x_\delta}\mathbb P_{\lambda,\pi}(A)
       +\mathbb P_{\mu,\pi}(A,Z_N>x_\delta)\\
    &\le\delta^{\varepsilon/2}
       +\mathbb P_\mu\!\left(
           \max_{0\le k\le n_\delta}Z_k>x_\delta
       \right).
\end{aligned}
\label{eq:app-track-change-measure-split}
\end{equation}
By Eq.~(\ref{eq:app-track-alt-choice}),
\begin{equation*}
    \frac{d^2n_\delta}{2}
    \le(1-\varepsilon)(1+\varepsilon/4)\log(1/\delta)
    \le(1-3\varepsilon/4)\log(1/\delta).
\end{equation*}
Consequently,
\[
    \mathbb P_\mu\!\left(\max_{k\le n_\delta}Z_k>x_\delta\right)
    \le\mathbb P_\mu\!\left(
        \max_{k\le n_\delta}S_k>
        \frac{\varepsilon\log(1/\delta)}{4d}
    \right).
\]
For a Gaussian random walk, stopping
$\exp(\theta S_k-\theta^2k/2)$ at its first crossing of $S_k>a$ before
$n$, and optimizing at $\theta=a/n$, gives
\[
    \mathbb P\!\left(\max_{0\le k\le n}S_k>a\right)
    \le e^{-a^2/(2n)}.
\]
Since $d^2n_\delta\le2\log(1/\delta)$, it follows that
\begin{equation*}
\begin{aligned}
    \mathbb P_\mu\!\left(\max_{k\le n_\delta}Z_k>x_\delta\right)
    &\le\exp\!\left(-\frac{\varepsilon^2\log^2(1/\delta)}
                              {32d^2n_\delta}\right)\\
    &\le\delta^{\varepsilon^2/64}.
\end{aligned}
\end{equation*}
Combining with Eq.~(\ref{eq:app-track-change-measure-split}) and using
$\delta^{\varepsilon/2}\le\delta^{\varepsilon^2/64}$ proves the bound
for $A$. Finally,
$\{N<2(1-\varepsilon)\log(1/\delta)/\Delta^2\}
\subseteq\{N\le n_\delta\}$.
All constants are independent of the policy and budget; finiteness of
the set of good pairs gives a common sufficiently small $\delta$.
\end{proof}

\begin{proof}[Proof of Theorem~\ref{thm:track-coverage-converse}]
Fix $b>0$ and $\varepsilon\in(0,1)$, let
$B_\delta=b\log(1/\delta)$, and
write $\bar k=\overline K_\varepsilon(b)$. For any
$\pi\in\Pi_\delta(B_\delta)$, let
\[
    \mathcal B_\pi
    =\{\widehat{\mathcal S}_\pi(B_\delta)\nsubseteq\mathcal G\}
\]
and
\[
    \mathcal R_\pi
    =\bigcup_{q\in\mathcal G,m\in[M]}
       \left\{q\in\widehat{\mathcal S}_\pi(B_\delta),\;
       N_{q,m}(\TS)<
       \frac{2(1-\varepsilon)}{\Delta_{q,m}^2}
       \log(1/\delta)\right\}.
\]
Reliability and Lemma~\ref{lem:track-necessary-evidence} imply that
\begin{equation*}
    \mathbb P_{\mu,\pi}(\mathcal B_\pi)\le\delta,
    \quad
    \mathbb P_{\mu,\pi}(\mathcal R_\pi)
    \le2|\mathcal G|M\delta^{\varepsilon^2/64}.
\end{equation*}
On the complement of these two events, all returned arms are good and
every required attribute has at least its stated number of observations.
If $k$ arms are returned, then $k\le|\mathcal G|$ and
\begin{equation*}
\begin{aligned}
    b\log(1/\delta)=B_\delta
    &\ge\sum_{q\in\widehat{\mathcal S}_\pi(B_\delta)}
         \sum_{m=1}^M c_mN_{q,m}(\TS)\\
    &\ge(1-\varepsilon)\log(1/\delta)
         \sum_{q\in\widehat{\mathcal S}_\pi(B_\delta)}h_q\\
    &\ge(1-\varepsilon)T_k^*(\mu)\log(1/\delta).
\end{aligned}
\end{equation*}
After division by $\log(1/\delta)$, this says
$(1-\varepsilon)T_k^*(\mu)\le b$, hence $k\le\bar k$ by definition.
The argument also covers $k=0$ since $T_0^*=0$. If $\bar k=|\mathcal G|$, any output
larger than $\bar k$ already contains a bad arm. Therefore
\begin{equation*}
\begin{aligned}
&\sup_{\pi\in\Pi_\delta(B_\delta)}
\mathbb P_{\mu,\pi}\!\left(
    |\widehat{\mathcal S}_\pi(B_\delta)|>
    \overline K_\varepsilon(b)
\right)\\
&\quad\le\delta+2|\mathcal G|M\delta^{\varepsilon^2/64}\longrightarrow0.
\end{aligned}
\end{equation*}
For $|\mathcal G|=0$, the same conclusion follows directly from reliability.
\end{proof}

\end{document}

%% file: math_commands.tex
\usepackage{amsmath,amsfonts,bm}

\def\eqref#1{equation~\ref{#1}}
\def\1{\bm{1}}

\DeclareMathAlphabet{\mathsfit}{\encodingdefault}{\sfdefault}{m}{sl}
\SetMathAlphabet{\mathsfit}{bold}{\encodingdefault}{\sfdefault}{bx}{n}